\documentclass[conference]{IEEEtran}

\usepackage[top=0.7
5in,bottom=0.75in,left=0.75in,right=0.75in]{geometry}
\usepackage{balance}

\usepackage[numbers,compress]{natbib}
\usepackage{multicol}
\usepackage[bookmarks=true]{hyperref}
\usepackage{booktabs} 
\usepackage{graphicx} 
\usepackage{gensymb}  
\usepackage{algorithm}
\usepackage[noend]{algpseudocode}
\usepackage{amsmath}
\usepackage{amsthm}
\usepackage{amssymb}
\usepackage{amsfonts}

\makeatletter
\let\MYcaption\@makecaption
\makeatother

\usepackage[font=footnotesize]{subcaption}

\makeatletter
\let\@makecaption\MYcaption
\makeatother

\usepackage{color}
\usepackage{bm}

\usepackage{arydshln}  
\makeatletter
  \renewcommand*\env@matrix[1][*\c@MaxMatrixCols c]{%
    \hskip -\arraycolsep
    \let\@ifnextchar\new@ifnextchar
  \array{#1}}
\makeatother

\usepackage{tikz}
\usetikzlibrary{math,positioning,automata}
\newcommand\centerofmass{%
    \tikz[radius=0.4em] {%
        \fill (0,0) -- ++(0.4em,0) arc [start angle=0,end angle=90] -- ++(0,-0.8em) arc [start angle=270, end angle=180];%
        \draw (0,0) circle;%
    }%
}

\newtheorem{theorem}{Theorem}
\newtheorem{lemma}{Lemma}
\newtheorem{remark}{Remark}
\newtheorem{definition}{Definition}

\IEEEoverridecommandlockouts  

\begin{document}

\title{
~\\  
Formal Connections between Template and Anchor Models via Approximate Simulation}

\author{Vince~Kurtz, Rafael~Rodrigues~da~Silva, Patrick~M.~Wensing, and~Hai~Lin%
\thanks{The partial support of the National Science Foundation (Grant No. ECCS-1253488, IIS-1724070, CNS-1830335, CMMI-1835186) is gratefully acknowledged.}%
\thanks{V. Kurtz, R. Silva, and H. Lin are with the Department of Electrical Engineering, University of Notre Dame, Notre Dame, IN, 46556 USA. \texttt{\{vkurtz,rrodri17,hlin1\}@nd.edu}.}%
\thanks{P.M. Wensing is with the Department of Aerospace and Mechanical Engineering, University of Notre Dame, Notre Dame, IN, 46556 USA. \texttt{pwensing@nd.edu}.}
}

\maketitle

\IEEEpeerreviewmaketitle

\begin{abstract}
    Reduced-order template models like the Linear Inverted Pendulum (LIP) and Spring-Loaded Inverted Pendulum (SLIP) are widely used tools for controlling high-dimensional humanoid robots. However, connections between templates and whole-body models have lacked formal underpinnings, preventing formal guarantees when it comes to integrated controller design. We take a small step towards addressing this gap by considering the notion of \textit{approximate simulation}. Derived from simulation relations for discrete transition systems in formal methods, approximate similarity means that the outputs of two systems can remain $\epsilon$-close. In this paper, we consider the case of controlling a balancer via planning with the LIP model. We show that the balancer approximately simulates the LIP and derive linear constraints that are sufficient conditions for maintaining ground contact. This allows for rapid planning and replanning with the template model by solving a quadratic program that enforces contact constraints in the full model. We demonstrate the efficacy of this planning and control paradigm in a simulated push recovery scenario for a planar 4-link balancer.
\end{abstract}

\section{Introduction}

Template models like the Linear Inverted Pendulum (LIP), Spring-Loaded Inverted Pendulum (SLIP), and Compass-Gait Walker are widely used for analysis, planning, and control of locomotion and balance \cite{wieber2016modeling,Geyer2018,kajita20013d,chen2012analysis,wensing2013high,posa2017balancing}. These models capture important properties of the full system, or anchor model, while also being simple enough to enable efficient planning. Often, the template model is used to generate a plan via Model Predictive Control (MPC) (e.g., \cite{kajita2003biped,wieber2006trajectory}) while the anchor model tracks this nominal trajectory, subject to contact constraints. The template-anchor control paradigm works well in practice, and is widely accepted in the robotics community~\cite{wieber2016modeling}. Furthermore, it has been hypothesized that biological systems use templates as sensorimotor control targets: animals may regulate their center-of-mass (CoM) to follow a lower-order model \cite{full1999templates}. 

However, a key weakness of the template-anchor paradigm is a lack of formal correctness. So far, formal connections between template and anchor models have not been identified. Beyond the insights that might come from identifying such connections, it is also difficult to provide provably correct whole-body controllers that track the template model while also accounting for constraints on contact forces, joint limits, joint torques, etc. 

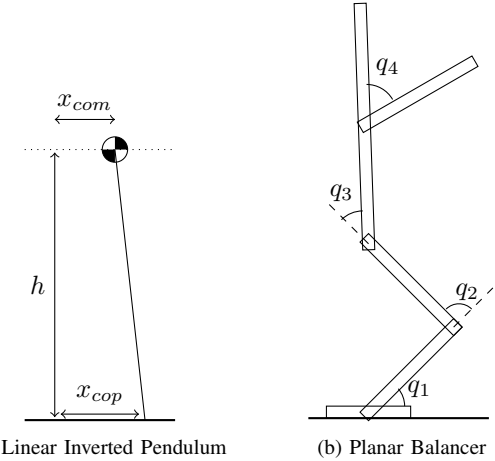
\begin{figure}
    \centering
    \begin{subfigure}[b]{0.45\linewidth}
        \centering
        \begin{tikzpicture}[scale=0.8]
            \draw[thick] (-0.5,0) -- (2,0);
            \draw[dotted] (-0.5,4.5) -- (2,4.5);
            \node at (1,4.5) (com) {\large{\centerofmass}};
            \draw (1,4.5) -- (1.5,0);
            
            \draw[<->] (0,0.05) -- node[left] {$h$} (0,4.45);
            \draw[<->] (0.1,0.1) -- node[above] {$x_{cop}$} (1.4,0.1);
            \draw[<->] (0,5.0) -- node[above] {$x_{com}$} (1,5.0);
        \end{tikzpicture}
        \caption{Linear Inverted Pendulum}
        \label{fig:lip}
    \end{subfigure}
    \begin{subfigure}[b]{0.45\linewidth}
        \centering
        \begin{tikzpicture}[scale=0.8]
            \tikzmath{
                \l = 2; 
                \w = 0.2; 
            }
            \draw[thick] (-1.0,-0.1) -- (2,-0.1);
            \draw (-0.7,-0.1) rectangle (0.7,0.1);
            \draw[rotate around={45:(0,0)}] (0-\w/2,0-\w/2) rectangle (\l+\w/2,0+\w/2);
            \draw[rotate around={135:({sqrt(\l)},{sqrt(\l)})}] ({sqrt(\l)-\w/2},{sqrt(\l)-\w/2}) rectangle ({\l+sqrt(\l)+\w/2},{sqrt(\l)+\w/2});
            \draw[rotate around={2:(0,2.8)}] (-0.1,2.7) rectangle (0.1,6.8);
            \draw[rotate around={-60:(-0.0,4.7)}] (-0.2,4.6) rectangle (0,6.8);
            \draw[dashed] ({sqrt(\l)},{sqrt(\l)}) -- ({sqrt(\l)+0.7},{sqrt(\l)+0.7});
            \draw[dashed] (0,2.8) -- (-0.7,3.5);
            \draw (0.6,0.1) arc (0:40:0.45) node[right] {$q_1$} ;
            \draw ({sqrt(\l)+0.3},{sqrt(\l)+0.3}) arc (45:135:0.3) node[above right] {$q_2$};
            \draw (-0.45,3.2) arc (135:90:0.5) node[above left] {$q_3$};
            \draw (0.4,5.2) arc (30:90:0.5) node [above right] {$q_4$};
        \end{tikzpicture}
        \caption{Planar Balancer}
        \label{fig:balancer}
    \end{subfigure}
    \caption{Template (linear inverted pendulum) and anchor (balancer in single support) models considered in this paper. We show that the balancer approximately simulates the LIP, which allows us to bound the tracking error between the two models.}
    \label{fig:models}
\end{figure}

The state-of-the-art in this regard is to formulate the tracking problem as a Quadratic Program (QP), which partial-feedback linearizes the system dynamics while also encoding contact, friction, and torque constraints \cite{herzog2014balancing,escande2014hierarchical,wensing2013generation}. Such QPs can be solved at kHz rates for high-dimensional rigid-body models, and have been applied to great effect in the control of humanoid robots \cite{kuindersma2016optimization,feng2015optimization,dai2014whole}. But this approach does not admit formal guarantees: we cannot ensure that the solver will always find a solution, or that the optimal solution will cause the anchor to converge to the template trajectory (though this property holds empirically under normal conditions with appropriately tuned controllers).

An alternative, more formal, approach is to restrict the anchor model to operate on a lower-dimensional manifold that is diffeomorphic to the template dynamics. This approach was taken by Poulakakis and Grizzle \cite{poulakakis2009spring} in the context of an asymmetric hopper wherein the SLIP model was embedded as the Hybrid Zero Dynamics \cite{westervelt2003hybrid} of the monopod. This approach admits strong formal guarantees, but finding a whole-body controller that enforces this diffeomorphic relationship is difficult, especially for high-dimensional systems like humanoids. Furthermore, it is challenging to repurpose redundant degrees of freedom for other tasks (e.g., holding a cup of coffee while walking) under this approach.

In this work, we present an alternative method of providing formal guarantees for control with template models. Specifically, we consider the paradigm of \textit{approximate simulation} \cite{girard2011approximate}. This concept is a generalization of simulation relations from formal methods \cite{baier2008principles} to continuous systems. If the anchor approximately simulates the template, then there exists a controller which tracks any trajectory of the template model with $\epsilon$ precision. Such a controller is known as an \textit{interface}. For any approximate simulation relation, there is also a \textit{simulation function}, a Lyapunov-like function that bounds the error between the two models. 

The contributions of this paper are as follows. First, we show that classical whole-body controllers based on task-space feedback linearization function as an interface, which proves that a balancer in single support approximately simulates a LIP. However, this simulation relation ignores important constraints that exist in whole-body robot models (e.g., contact constraints, joint limits, torque limits, etc.). As a second contribution, we introduce a new MPC scheme that uses the simulation relationship to address contact wrench cone (CWC) \cite{pang2000stability} constraints during planning with the LIP. Inspired by centroidal momentum planning methods \cite{dai2016planning,audren20183}, we address non-convexity in this MPC scheme by introducing a convex inner approximation to the CWC constraints. This inner approximation allows us to perform MPC for the LIP model by solving a QP that already accounts for contact constraints in the whole-body model, as shown in Figure \ref{fig:control_flow}. In contrast to existing work \cite{caron2016multi}, this new method does not require an admissible region of the CoM to be specified a-priori.  

The remainder of this paper is organized as follows: Section \ref{sec:background} presents a formal definition of approximate simulation and introduces relevant results for linear systems. Section~\ref{sec:problem_form} provides system definitions and a problem statement. Section~\ref{sec:results} presents our principal results, which are illustrated with a simulated push-recovery scenario in Section~\ref{sec:example}. We provide a brief discussion in Section~\ref{sec:discuss}, and conclude with Section~\ref{sec:conclusion}.

\section{Background}\label{sec:background}

\subsection{Approximate Simulation}\label{subsec:approximate_simulation}
Approximate simulation is formally defined in terms of two systems $\Sigma_1$ and $\Sigma_2$:
\begin{equation}
    \Sigma_1 :
    \begin{cases}
        \dot{\mathbf{x}}_1 = f_1(\mathbf{x}_1,\mathbf{u}_1) \\
        \mathbf{y}_1 = g_2(\mathbf{x}_1)
    \end{cases},~~~
    \Sigma_2 :
    \begin{cases}
        \dot{\mathbf{x}}_2 = f_2(\mathbf{x}_2,\mathbf{u}_2) \\
        \mathbf{y}_2 = g_2(\mathbf{x}_2)
    \end{cases},
\end{equation}
where $\mathbf{x}_i \in \mathbb{R}^{n_i}$ are the system states, $\mathbf{u}_i \in \mathbb{R}^{p_i}$ are the control inputs, and $\mathbf{y}_i \in \mathbb{R}^m$ are the system outputs. Note that the states may be different sizes but the outputs---which in our case correspond to the task-space---must be the same size. Without loss of generality, we consider $\Sigma_1$ to be the whole-body (anchor) model and $\Sigma_2$ to be the reduced-order (template) model. 

Approximate simulation for continuous systems \cite{girard2011approximate} is defined in terms of a Lyapunov-like simulation function $\mathcal{V}$ and an interface function $u_\mathcal{V}$:
\begin{definition}
    Let $\mathcal{V} : \mathbb{R}^{n_2} \times \mathbb{R}^{n_1} \to \mathbb{R}^+$ be a smooth function and $u_\mathcal{V} : \mathbb{R}^{p_2} \times \mathbb{R}^{n_1} \times \mathbb{R}^{n_2} \to \mathbb{R}^{p_1}$ be a continuous function. $\mathcal{V}$ is a simulation function of $\Sigma_2$ by $\Sigma_1$ and $u_\mathcal{V}$ is an associated interface if there exists a class-$\kappa$ function\footnote{A function $\gamma : \mathbb{R}^+ \to \mathbb{R}^+$ is a class-$\kappa$ function if it is continuous, strictly increasing, and $\gamma(0)=0$.} $\gamma$ such that for all $\mathbf{x}_1, \mathbf{x}_2 \in \mathbb{R}^{n_1} \times \mathbb{R}^{n_2}$,
    \begin{equation}\label{eq:output_error_bounded}
        \mathcal{V}(\mathbf{x}_1,\mathbf{x}_2) \geq \|g_1(\mathbf{x}_1)-g_2(\mathbf{x}_2)\|
    \end{equation}
    and for all $\mathbf{u}_2 \in \mathbb{R}^{p_2}$ satisfying $\gamma(\|\mathbf{u}_2\|) < \mathcal{V}(\mathbf{x}_1, \mathbf{x}_2)$,
    \begin{equation}\label{eq:simulation_fcn_decay}
        \frac{\partial\mathcal{V}}{\partial\mathbf{x}_2}f_2(\mathbf{x}_2, \mathbf{u}_2) +\frac{\partial\mathcal{V}}{\partial\mathbf{x}_1}f_1(\mathbf{x}_1, u_{\mathcal{V}}(\mathbf{u}_2, \mathbf{x}_1, \mathbf{x}_2)) < 0.
    \end{equation}
\end{definition}
These conditions essentially state that when the interface is applied, the simulation function always bounds the output error (\ref{eq:output_error_bounded}) and decreases as long as $\mathbf{u}_2$ is not too large (\ref{eq:simulation_fcn_decay}). 
\begin{definition}[\cite{girard2011approximate}]
    $\Sigma_1$ approximately simulates $\Sigma_2$ if and only if there exists a simulation function $\mathcal{V}$ of $\Sigma_2$ by $\Sigma_1$.
\end{definition}

\begin{figure}
    \centering
    \begin{tikzpicture}
        \node[minimum width=2cm, minimum height=0.9cm] (planner) at (0,5) {Cost Function};
        \node[minimum width=2cm] (constraints) at (1.5,6) {Contact Constraints};
        \node[draw, minimum width=2cm, right=2 of planner, minimum height=0.7cm] (template) {Template MPC};
        \node[right=0.3 of template] (frequency) {\textit{QP @20Hz}};
        \node[draw, minimum width=2cm, minimum height=0.7cm, below=0.7 of template] (interface) {Interface};
        \node[draw, minimum width=2cm, minimum height=0.7cm, below=0.7 of interface] (linearization) {Feedback Linearization};
        \node[draw, minimum width=2cm, minimum height=0.7cm, below=0.7 of linearization] (robot) {Robot};
        
        \draw[->,thick,dashed] (planner) -- (template);
        \draw[->,thick,dashed] (constraints) |- ([yshift=0.25cm]template.west);
        \draw[->,thick] (template) --  node[right] {$\mathbf{x}_{lip},u_{lip}$} (interface);
        \draw[->,thick] (interface) --  node[right] {$\mathbf{u}_{task}$} (linearization);
        \draw[->,thick] (linearization) -- node[right] {$\bm{\tau}$} (robot);
        
        \draw[->,thick] (robot) --  ++(-3,0) -- node[left] {$\mathbf{q},\mathbf{\dot{q}}$} ++(0,4.0) -- ([yshift=-0.25cm]template.west);
        \draw[<-,thick] (interface.west) -- ++(-2,0);
        \draw[<-,thick] (linearization.west) -- ++(-1.2,0);
    \end{tikzpicture}
    \caption{Control flow for our approach. While most whole-body controllers account for contact constraints in the feedback linearization phase with a QP, we project contact constraints back to the template model using the approximate simulation relation. Note that $(\mathbf{q},\dot{\mathbf{q}})$ are used only to compute the task state $\mathbf{x}_{task}$ in the Template MPC block.}
    \label{fig:control_flow}
\end{figure}
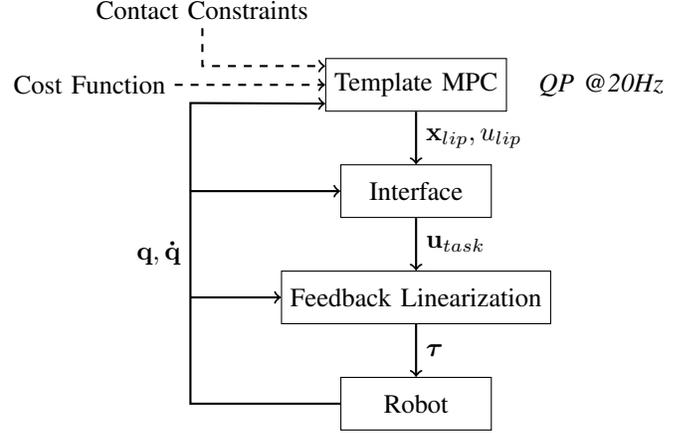

If $\Sigma_1$ approximately simulates $\Sigma_2$, we can use the simulation function to bound the output error of the two systems by $\epsilon$:
\begin{theorem}[\cite{girard2009hierarchical}]\label{theorem:bound}
    Let $\mathcal{V}$ be a simulation function of $\Sigma_2$ by $\Sigma_1$ and $u_\mathcal{V}$ be an associated interface. Let $\mathbf{u}_2(t)$ be an admissible input of $\Sigma_2$ with associated state and output trajectories $\mathbf{x}_2(t)$ and $\mathbf{y}_2(t)$. Let $\mathbf{x}_1(t)$ be a state trajectory of $\Sigma_1$ satisfying
    \begin{equation*}
        \dot{\mathbf{x}}_1 = f_1(\mathbf{x}_1, u_\mathcal{V}(\mathbf{u}_2, \mathbf{x}_1, \mathbf{x}_2))
    \end{equation*}
    and $\mathbf{y}_1(t)$ be the associated output trajectory. Then 
    \begin{equation}
        \|\mathbf{y}_1(t)-\mathbf{y}_2(t)\| \leq \epsilon
    \end{equation}
    where
    \begin{equation}\label{eq:epsilon}
        \epsilon = \max\big\{ \mathcal{V}(\mathbf{x}_1(0),\mathbf{x}_2(0)), \gamma(\|\mathbf{u}_2\|_\infty)\big\}.
    \end{equation}
\end{theorem}

Finding a simulation function for two arbitrary dynamical systems is a difficult and open problem, though some promising results with sum-of-squares programming exist \cite{girard2005approximate,murthy2015computing}. For linear systems, however, there are well-defined conditions for the existence of a simulation function \cite{girard2009hierarchical}. We summarize these conditions in the following subsection.

\subsection{Approximate Simulation for Linear Systems}\label{subsec:linear_simulation}

Consider the case when both the template and the anchor are linear systems, i.e.,
\begin{equation}
    \Sigma_i :
    \begin{cases}
        \dot{\mathbf{x}}_i = \mathbf{A}_i\mathbf{x}_i + \mathbf{B}_i\mathbf{u}_i \\
        \mathbf{y}_i = \mathbf{C}_i\mathbf{x}_i
    \end{cases}, ~~~ i = \{1,2\}.
\end{equation}
In this case, there are strong results regarding whether a simulation relation exists. First note the following Lemma:
\begin{lemma}[\citet{girard2007approximate}]\label{lemma:M}
    If the anchor system $\Sigma_1$ is stabilizable with feedback gain $\mathbf{K}$, i.e., $(\mathbf{A}_1 + \mathbf{B}_1\mathbf{K})$ is Hurwitz, then there exists a positive definite symmetric matrix $\mathbf{M}$ and positive scalar constant $\lambda$ such that the following hold:
    \begin{gather}
        \mathbf{M} \geq \mathbf{C}_1^T\mathbf{C}_1, \\
        (\mathbf{A}_1+\mathbf{B}_1\mathbf{K})^T\mathbf{M} + \mathbf{M}(\mathbf{A}_1+\mathbf{B}_1\mathbf{K}) \leq -2 \lambda \mathbf{M}.
    \end{gather}
\end{lemma}
Such an $\mathbf{M}$ can be used to show exponential convergence of $\mathbf{y}_1$ to zero with rate $\lambda$ under the feedback $\mathbf{u}_1 = \mathbf{K}\mathbf{x}_1$. Note that these conditions are linear matrix inequalities in $\mathbf{M}$. This means that given $\mathbf{K}$ and $\lambda$, $\mathbf{M}$ can be computed using semidefinite programming.

We can now state the following Theorem:

\begin{theorem}[\citet{girard2009hierarchical}]\label{theorem:simulation}
    Assume that $\Sigma_1$ is stabilizable with feedback gain $\mathbf{K}$ and that there exist matrices $\mathbf{P}$ and $\mathbf{Q}$ such that the following conditions hold:
    \begin{align}
        \mathbf{P}\mathbf{A}_2 &= \mathbf{A}_1\mathbf{P} + \mathbf{B}_1\mathbf{Q}, \\
        \mathbf{C}_2 & = \mathbf{C}_1\mathbf{P}.
    \end{align}
    Then a simulation function of $\Sigma_2$ by $\Sigma_1$ is given by
    \begin{equation}
        \mathcal{V}(\mathbf{x}_1,\mathbf{x}_2) = \sqrt{ (\mathbf{x}_1-\mathbf{P}\mathbf{x}_2)^T\mathbf{M}(\mathbf{x}_1-\mathbf{P}\mathbf{x}_2}),
    \end{equation}
    an associated interface is
    \begin{equation}
        u_{\mathcal{V}} = \mathbf{R}\mathbf{u}_2 + \mathbf{Q}\mathbf{x}_2 + \mathbf{K}(\mathbf{x}_1 - \mathbf{P}\mathbf{x}_2),
    \end{equation}
    and the class-$\kappa$ function $\gamma$ is given by
    \begin{equation}\label{eq:gamma_fcn}
        \gamma(\nu) = \frac{\|\sqrt{\mathbf{M}}(\mathbf{B}_1\mathbf{R}-\mathbf{P}\mathbf{B}_2)\|}{\lambda}\nu,
    \end{equation}
    where $\mathbf{R}$ is an arbitrary matrix of proper dimensions. 
\end{theorem}

The matrix $\mathbf{R}$ acts as a ``feedforward'' mapping from $\mathbf{u}_2$ to $\mathbf{u}_1$. While the simulation relation holds for any $\mathbf{R}$ of proper dimensions, choosing $\mathbf{R}$ to minimize (\ref{eq:gamma_fcn}) is a logical choice, as this tightens the error bound $\epsilon$ (\ref{eq:epsilon}).

\section{Problem Formulation}\label{sec:problem_form}

\subsection{System Definitions}
In this paper, we consider controlling a balancer in single support using the LIP model as a template. 

An example of a balancer used as the anchor model is shown in Figure \ref{fig:balancer}. We assume that all joints are actuated and that the system is mounted to a single foot in contact. By modeling the balancer as a kinematic tree, we can write the equations of motion with minimal coordinates in ``manipulator'' form:
\begin{equation}\label{eq:balancer}
    \mathbf{H}(\mathbf{q})\ddot{\mathbf{q}} + \mathbf{C}(\mathbf{q},\dot{\mathbf{q}})\dot{\mathbf{q}} + \bm{\tau}_g = \bm{\tau}
\end{equation}
where $\mathbf{q}$ are joint angles, $\mathbf{H}(\mathbf{q})$ is the mass matrix, $\mathbf{C}(\mathbf{q},\mathbf{\dot{q}})$ accounts for Coriolis and centripetal terms, $\bm{\tau}_g$ is torque due to gravity, and $\bm{\tau}$ are applied torques. The gravitational vector is denoted $\mathbf{g}$ and has magnitude $g$. The total mass of the robot is $m$. Note that  (\ref{eq:balancer}) is valid only if contact constraints are not violated. 

We consider the CoM position $\mathbf{p}_G \in \mathbb{R}^3$ and centroidal momentum $\mathbf{h}_G = [\mathbf{k}_G^T ~ \mathbf{l}^T_G]^T \in \mathbb{R}^6$ as the task-space of the anchor model:
\begin{equation*}
    \mathbf{x}_{task} = 
    \begin{bmatrix}
        \mathbf{p}^T_{G} & \mathbf{h}^T_{G}
    \end{bmatrix}^T,
\end{equation*}
where $\mathbf{k}_G$ is the angular momentum about the CoM and $\mathbf{l}_G$ is the net linear momentum. The task-space dynamics can be computed in terms of the centroidal dynamics \cite{orin2013centroidal,wensing2016improved}:
\begin{gather}
    \dot{\mathbf{p}}_G = \frac{1}{m}\mathbf{l}_G, \\
    \mathbf{h}_{G} = \mathbf{A}_{G}(\mathbf{q})\dot{\mathbf{q}}, \\
    \dot{\mathbf{h}}_{G} = \mathbf{A}_{G}(\mathbf{q})\,\ddot{\mathbf{q}} + \dot{\mathbf{A}}_{G}(\mathbf{q},\dot{\mathbf{q}})\,\dot{\mathbf{q}}.
\end{gather}

The LIP model \cite{kajita20013d}, shown in Figure \ref{fig:lip}, is constrained to the ($x,z$) plane and governed by the horizontal position of the CoM $x_{com}$ and the center of pressure $x_{cop}$. The vertical position of the CoM is fixed at height $h$ at all times. The dynamics of the LIP model are given by
\begin{equation}
    \ddot{x}_{com} = \omega^2(x_{com}-x_{cop}),
\end{equation}
where $\omega = \sqrt{\frac{g}{h}}$ is the natural frequency of the LIP. 

For easier comparison with the task-space of the anchor model, we can write the LIP dynamics as a linear system
\begin{equation}\label{eq:lip}
    \dot{\mathbf{x}}_{lip} = \mathbf{A}_{lip}\mathbf{x}_{lip} + \mathbf{B}_{lip}u_{lip},
\end{equation}
where
\begin{equation*}
    \mathbf{x}_{lip} = 
    \begin{bmatrix}[ccc;{2pt/2pt}ccc;{2pt/2pt}ccc]
        x_{com} & 0 & h & 0 & 0 & 0 & m\dot{x}_{com} & 0 & 0 
    \end{bmatrix}^T
\end{equation*}
expresses the CoM position and spatial momentum of the LIP model and $u_{lip} = x_{cop}$ is the $x$-position of the center of pressure. While it may seem odd to include static elements like angular momentum in the LIP model, this particular definition of $\mathbf{x}_{lip}$ will allow us to draw a connection between an interface that certifies approximate simulation and the PD control law often used to track template models.

\subsection{Problem Statement}

As per Section \ref{subsec:approximate_simulation}, finding an interface and a simulation function allows us to bound the output error between two systems. With this in mind, our primary goal is to certify that the balancer approximately simulates the LIP.

More formally, we define the output of the template (LIP) model to be the full system state, i.e.,
\begin{equation*}
    \mathbf{y}_{lip} = \mathbf{x}_{lip},
\end{equation*}
and the output of the anchor model to be the task-space (position and spatial momentum of the CoM):
\begin{equation*}
    \mathbf{y}_{task} = \mathbf{x}_{task} = \begin{bmatrix} \mathbf{p}^T_{G} & \mathbf{h}^T_{G} \end{bmatrix}^T.
\end{equation*}
Our goal is to find a simulation function and an interface that certify that the anchor model (\ref{eq:balancer}) approximately simulates the template model (\ref{eq:lip}).

Finding such an interface and simulation function would allow us to guarantee that the CoM of the balancer tracks the LIP's CoM and bound the associated tracking error, subject to constraints on contact, torques, joint limits, etc. This leads us to our secondary goal, which is to use the simulation relation to project contact constraints for the balancer to constraints on the LIP model. That way, when we plan with the LIP model, we can be sure that the balancer will maintain ground contact. 

\section{Theoretical Results}\label{sec:results}

\subsection{Approximate Simulation for the LIP and Balancer}\label{subsec:approximate_sim_for_balancer}

In order to harness the results presented in Section \ref{subsec:linear_simulation}, we take a task-space feedback linearization of the balancer. To do so, we define a task-space interia matrix
\begin{equation}
    \bm{\Lambda} = (\mathbf{A}_{G}\mathbf{H}^{-1}\mathbf{A}_{G}^T)^{-1}
\end{equation}
and apply torques $\bm{\tau}$ such that
\begin{equation}\label{eq:feedback_linearization}
    \bm{\tau} = \mathbf{A}_{G}^T\bm{\Lambda}(\mathbf{u}_{task} - \dot{\mathbf{A}}_{G}\dot{\mathbf{q}} + \mathbf{A}_{G} \mathbf{H}^{-1}(\mathbf{C}\dot{\mathbf{q}}+\bm{\tau}_g)),
\end{equation}
where $\mathbf{u}_{task} = \dot{\mathbf{h}}_G$ is a virtual control. This gives rise to the task-space dynamics
\begin{equation}\label{eq:anchor_linear}
    \dot{\mathbf{x}}_{task} = \mathbf{A}_{task}\mathbf{x}_{task} + \mathbf{B}_{task}\mathbf{u}_{task},
\end{equation}
where 
\begin{equation*}
    \mathbf{A}_{task} = 
    \begin{bmatrix}
        \mathbf{0}_{3\times6} & (1/m)\mathbf{I}_{3\times3} \\
        \mathbf{0}_{6\times6} & \mathbf{0}_{6\times3}
    \end{bmatrix}, ~~
    \mathbf{B}_{task} =   
    \begin{bmatrix}
        \mathbf{0}_{3\times6} \\
        \mathbf{I}_{6\times6}
    \end{bmatrix}.
\end{equation*}

In the case where $dim(\mathbf{q}) > dim(\mathbf{x}_{task})$, we can resolve redundancies in the standard manner via the null-space projector $N$, such that applying $\bm{\tau} + N\bm{\tau}_0$ has the same effect on the task-space as applying $\bm{\tau}$ alone \cite{khatib1987unified}. This allows us to design $\bm{\tau}_0$ to achieve secondary control objectives like reducing extraneous motion or controlling a certain limb. In our example, we designed $\bm{\tau}_0$ to regulate the balancer to the nominal pose shown in Figure \ref{fig:balancer}. 

Now we have a linear system that describes the evolution of the planar balancer's CoM (\ref{eq:anchor_linear}), and another linear system describing the evolution of the LIP's CoM (\ref{eq:lip}). Taking the outputs as described above, i.e., $\mathbf{C}_{task} = \mathbf{C}_{lip} = \mathbf{I}_{9\times9}$, we can use the results from Section \ref{subsec:linear_simulation} to find a simulation function and an associated interface. 

\begin{figure}
    \centering
    \begin{tikzpicture}
        \node[draw] (template) at (0,0) {Template Model};
        \node[draw, below=0.7 of template] (task) {Task-Space Model};
        \node[draw, below=0.7 of task] (anchor) {Whole-Body (Anchor) Model};
        
        \draw[<->,thick,dashed] (template) -- node[right] {\textit{approximate simulation}} (task);
        \draw[<->,thick] (task) -- node[right] {\textit{feedback linearization}} (anchor);
    \end{tikzpicture}
    \caption{Heirarchy of models used in our approach. We show that the whole body (anchor) model simulates the template model via a task-space feedback linearization.}
    \label{fig:my_label}
\end{figure}
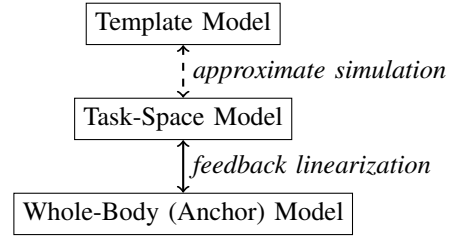

First, we find a feedback control gain matrix $\mathbf{K}$ to stabilize the centroidal dynamics (\ref{eq:anchor_linear}). There are many techniques to choose such a gain: LQR can be used to find a $\mathbf{K}$ that is optimal with respect to a certain cost functional. Given $\mathbf{K}$, we can use semi-definite programming to find $\mathbf{M}$ following Lemma \ref{lemma:M}. 

Then, following Theorem \ref{theorem:simulation}, we define the matrices $\mathbf{P}$, $\mathbf{Q}$, and $\mathbf{R}$ as follows:
\begin{equation*}
    \mathbf{P} = \mathbf{I}_{9\times9}, ~~~~ \mathbf{Q} = 
    \begin{bmatrix}
        \begin{bmatrix}\mathbf{0}_{3\times1}\\\omega^2\\\mathbf{0}_{2\times1}\end{bmatrix} \mathbf{0}_{6\times8}
    \end{bmatrix},~~~ \mathbf{R} = 
    \begin{bmatrix}
        \mathbf{0}_{3\times1} \\
        -\omega^2 \\
        \mathbf{0}_{2\times1}
    \end{bmatrix}.
\end{equation*}
This choice of $\mathbf{P}$ and $\mathbf{Q}$ satisfies the conditions of Theorem \ref{theorem:simulation}. This particular choice of $\mathbf{R}$ not only minimizes (\ref{eq:gamma_fcn}), but also establishes a connection between the interface and the PD controller often used in practice. 

To see this connection, consider a robot with mass $m=1$. A controller for the feedback-linearized anchor model is often designed as follows:
\begin{gather}
    \mathbf{u}_{task} = 
    \begin{bmatrix}
        -\mathbf{K}_{ang}\mathbf{k}_{G} \\
        \ddot{\mathbf{p}}_G
    \end{bmatrix}\label{eq:pd_final_control},\\
    \ddot{\mathbf{p}}_G = 
    \begin{bmatrix}
        \ddot{x}_{com} \\
        \mathbf{0}_{2\times1}
    \end{bmatrix} + \mathbf{K}_D(\dot{\mathbf{p}}_{lip}-\dot{\mathbf{p}}_G) + \mathbf{K}_P(\mathbf{p}_{lip}-\mathbf{p}_G),
\end{gather}
where $\mathbf{p}_{lip},\dot{\mathbf{p}}_{lip}$ are the CoM positions and velocities of the LIP, $\mathbf{K}_{ang}$ is an angular momentum damping term, and $\mathbf{K}_P$ and $\mathbf{K}_D$ are matrices of tuned control gains. Recalling that $\ddot{x}_{com} = \omega^2(x_{com}-u_{lip})$, we can rewrite the first term of (\ref{eq:pd_final_control}) as
\begin{align*}
    \begin{bmatrix}
        \mathbf{0}_{3\times1} \\
        \ddot{x}_{com} \\
        \mathbf{0}_{2\times1}
    \end{bmatrix}
    & = 
    \begin{bmatrix}
        \mathbf{0} \\
        \omega^2(x_{com}-u_{lip}) \\
        \mathbf{0}
    \end{bmatrix} \\
    & = \mathbf{R}u_{lip} + \mathbf{Q}\mathbf{x}_{lip}.
\end{align*}
Similarly, we can use the fact that $\mathbf{x}_{lip} = [\mathbf{p}_{lip}^T ~ \mathbf{0}_{3\times1}^T ~ \dot{\mathbf{p}}_{lip}^T]$ and $\mathbf{x}_{task} = [\mathbf{p}_{G}^T ~ \mathbf{k}_{G}^T ~ \dot{\mathbf{p}}_{G}^T]$ to rearrange the PD gain terms as
\begin{align*}
    \begin{bmatrix}
        \mathbf{K}_{ang}\mathbf{k}_G\\ 
        \mathbf{K}_D(\dot{\mathbf{p}}_{lip}-\dot{\mathbf{p}}_G) + \mathbf{K}_P(\mathbf{p}_{lip}-\mathbf{p}_G) 
    \end{bmatrix}
    = \mathbf{K}(\mathbf{x}_{task}-\mathbf{x}_{lip}),
\end{align*}
recovering the familiar form of the interface
\begin{equation}\label{eq:interface}
    \mathbf{u}_{task} = \mathbf{R}u_{lip} + \mathbf{Q}\mathbf{x}_{lip} + \mathbf{K}(\mathbf{x}_{task}-\mathbf{x}_{lip}).
\end{equation}
Thus the commonly used task-space PD controller is in fact a special case of an interface that admits an approximate simulation relation. 

Regardless of whether the stabilizing gain matrix $\mathbf{K}$ is generated as a PD controller or via techniques like LQR, the resulting simulation function takes the form 
\begin{equation}\label{eq:simulation_fcn}
    \mathcal{V}(\mathbf{x}_{task},\mathbf{x}_{lip}) = \sqrt{(\mathbf{x}_{task}-\mathbf{x}_{lip})^T\mathbf{M}(\mathbf{x}_{task}-\mathbf{x}_{lip})},
\end{equation}
where $\mathbf{M}$ can be computed with semi-definite programming as per Lemma \ref{lemma:M}. 

Recall from Theorem \ref{theorem:bound} that the simulation function $\mathcal{V}$ always bounds the output error, which in this case is $\|\mathbf{x}_{task}-\mathbf{x}_{lip}\|$. Furthermore, $\mathcal{V}$ always decreases, except when $\gamma(\|u_{lip}\|) \geq \mathcal{V}(\mathbf{x}_{task}, \mathbf{x}_{lip})$. From Theorem \ref{theorem:simulation}, we have
\begin{align*}
    \gamma(\nu) & = \frac{\|\sqrt{\mathbf{M}}(\mathbf{B}_{task}\mathbf{R}-\mathbf{P}\mathbf{B}_{lip})\|}{\lambda}\nu \\
     & = \frac{\|\sqrt{\mathbf{M}}(0)\|}{\lambda}\nu =0.
\end{align*}
In other words, $\mathcal{V}$ is decreasing along all trajectories of the LIP model, regardless of the value of $u_{lip}$. 
    
\subsection{Projecting Contact Constraints to the Template}

When we use this simulation relation to control the anchor system, we do have some restrictions. Intuitively, the CoM cannot follow arbitary trajectories: at the very least, it cannot reside beyond the limits imposed by a fixed ground contact. As an initial step toward handling such constraints, we show that contact constraints on the anchor model can be reformulated as linear constraints for MPC planning with the template model.

\begin{figure*}
    \centering
    \includegraphics[width=0.16\textwidth]{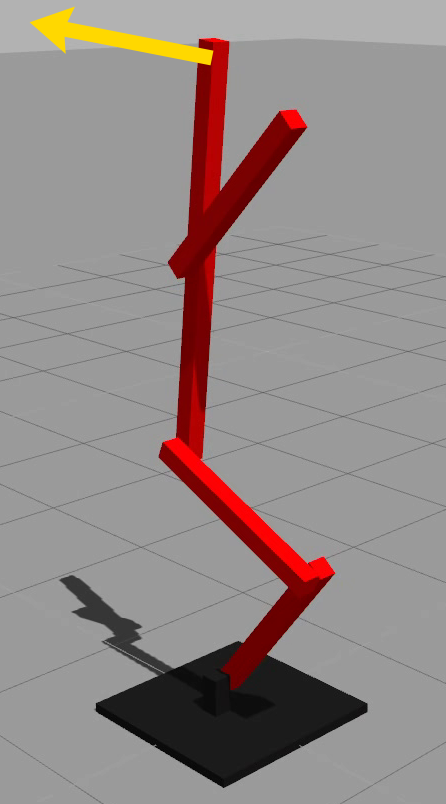}
    \includegraphics[width=0.16\textwidth]{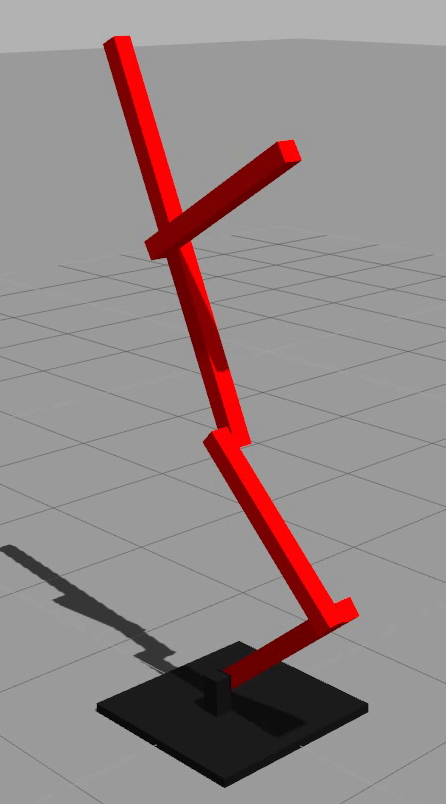}
    \includegraphics[width=0.16\textwidth]{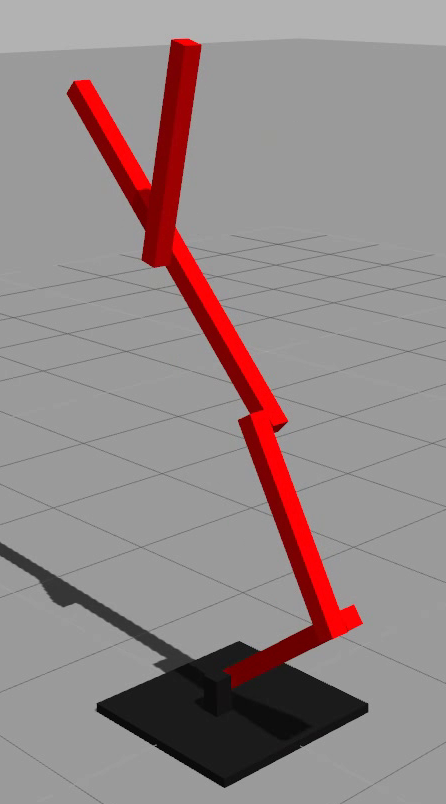}
    \includegraphics[width=0.16\textwidth]{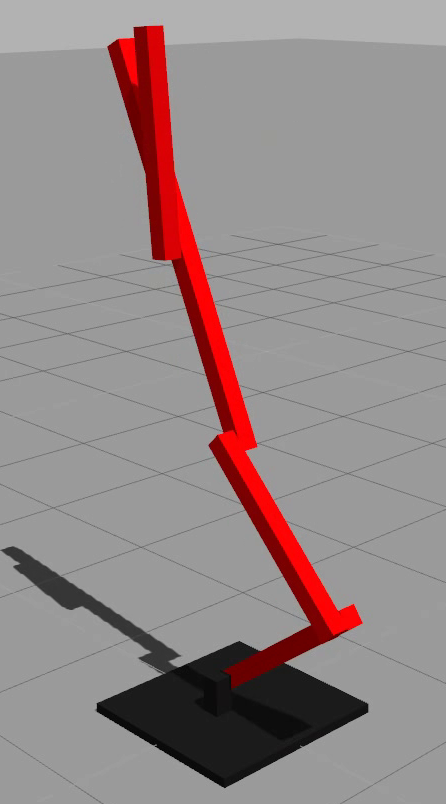}
    \includegraphics[width=0.16\textwidth]{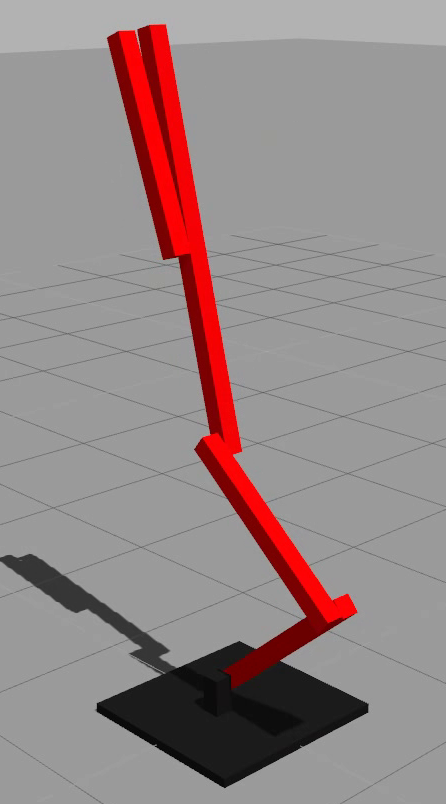}
    \includegraphics[width=0.16\textwidth]{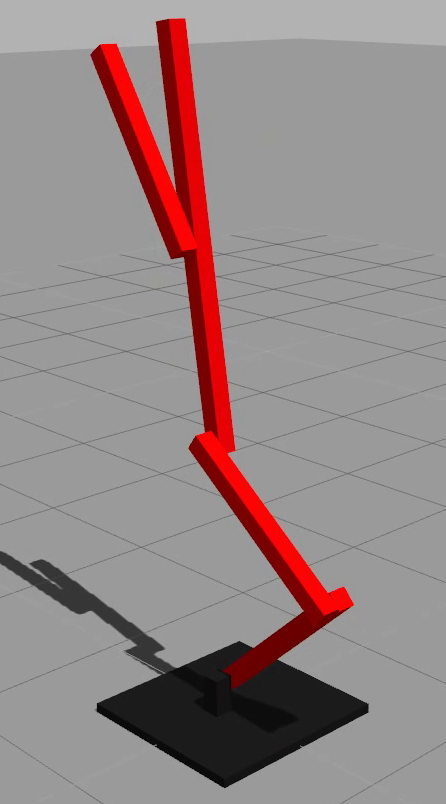}
    
    \caption{Snapshots taken at 1Hz as the multi-link balancer recovers from an initial push via planning with the LIP model. Using the approximate simulation relation between the balancer and the LIP allows us to guarantee that the balancer will track the LIP trajectory and maintain ground contact.}
    \label{fig:snapshots}
\end{figure*}

The Contact Wrench Cone (CWC) contact constraint states that all contact forces must remain in cones defined by a Coulomb friction model \cite{pang2000stability}. In terms of the spatial force expressed at the ground frame $\mathbf{f}_0$, the CWC can be expressed as
\begin{equation}
    CWC = \Big\{ \mathbf{f}_0 \mid \mathbf{f}_0 = \sum_j \begin{bmatrix}\mathbf{S}(\mathbf{p}_{c_j})\\\mathbf{I}\end{bmatrix}\bm{f}_{c_j}, \left \|\begin{bmatrix}f_{c_j}^x\\f_{c_j}^y\end{bmatrix}\right\| \leq \mu f_{c_j}^z \Big\},
\end{equation}
where $c_j$ are ground contacts, $\mathbf{p}_{c_j}$ are their associated positions in the ${0}$ frame, $\bm{f}_{c_j} \in \mathbb{R}^3$ are ground contact forces, $\mu$ is the coefficient of friction, and $\mathbf{S}(\cdot)$ is the skew-symmetric cross product matrix.

We can express the CWC constraint in terms of $\mathbf{u}_{task}=\dot{\mathbf{h}}_G$:
\begin{equation*}
    {}^0\mathbf{X}^*_G \left(\mathbf{u}_{task}-\begin{bmatrix}\mathbf{0}\\m\mathbf{g}\end{bmatrix}\right) \in CWC,
\end{equation*}
where ${}^0\mathbf{X}^*_G$ is the spatial force transform from the CoM frame $\{G\}$ to the ground frame $\{0\}$ given by
\begin{equation*}
    {}^0\mathbf{X}_G^* = 
    \begin{bmatrix}
        \mathbf{I} & \mathbf{S}(\mathbf{p}_G) \\
        \mathbf{0} & \mathbf{I}
    \end{bmatrix}.
\end{equation*}

If we consider friction pyramids as inner approximations of friction cones, we can derive a polytopic under-approximation of the CWC \cite{caron2015stability,caron2015leveraging}, i.e., $\mathbf{A} \mathbf{f}_0 \leq 0$. In this case, the CWC criterion can be written as a bilinear constraint on $\mathbf{x}_{task}$, $\mathbf{u}_{task}$
\begin{equation}\label{eq:cwc_bilinear}
    \mathbf{A} {}^0\mathbf{X}_G^*(\mathbf{u}_{task} - \begin{bmatrix}\mathbf{0}\\m\mathbf{g}\end{bmatrix}) \leq 0,
\end{equation}
where the bilinearity arises due to the dependence of ${}^0\mathbf{X}_G^*$ on a cross-product term of $\mathbf{p}_G$. 

However, if we constrain the CoM acceleration $\dot{\mathbf{l}}_G$, we can formulate the CWC criterion as a linear constraint on $\mathbf{u}_{task}$ and $\mathbf{x}_{task}$ \cite{audren20183}. This is shown by the following Theorem:

\begin{theorem}
    If the CoM acceleration $\dot{\mathbf{l}}_G$ of the anchor model is constrained by $\|\dot{\mathbf{l}}_G\|_{\infty} \leq \dot{l}_{max}$, then there exists a linear constraint $\mathbf{A}_{cwc}\begin{bmatrix}\mathbf{x}_{task} \\ \mathbf{u}_{task}\end{bmatrix} \leq \mathbf{b}_{cwc}$ that is a sufficient condition for the CWC contact criterion.
\end{theorem}
\begin{proof}
    We will prove by construction. First, recall that $\mathbf{a} \times \mathbf{b} = \mathbf{S}(\mathbf{a})\mathbf{b} = -\mathbf{S}(\mathbf{b})\mathbf{a}$.
    With this in mind, the CWC criterion (\ref{eq:cwc_bilinear}) can be written as
    \begin{gather}
        \mathbf{A} {}^0\mathbf{X}_G^*\dot{\mathbf{h}}_{G} \leq \mathbf{A} {}^0\mathbf{X}_G^* \begin{bmatrix}\mathbf{0}\\m\mathbf{g}\end{bmatrix} \\[1ex]
        \mathbf{A} \begin{bmatrix}\mathbf{I} \\ \mathbf{0}\end{bmatrix}\dot{\mathbf{k}}_G + \mathbf{A}\begin{bmatrix}\mathbf{S}(\mathbf{p}_G) \\ \mathbf{I}\end{bmatrix}\dot{\mathbf{l}}_G \leq \mathbf{A}\begin{bmatrix}\mathbf{I} & \mathbf{S}(\mathbf{p}_G) \\ \mathbf{0} & \mathbf{I}\end{bmatrix}\begin{bmatrix}\mathbf{0}\\m\mathbf{g}\end{bmatrix} \\[1ex]
        \mathbf{A}\dot{\mathbf{h}}_G + \mathbf{A}\begin{bmatrix}\mathbf{S}(m\mathbf{g})-\mathbf{S}(\dot{\mathbf{l}}_G) \\ \mathbf{0}\end{bmatrix}\mathbf{p}_G \leq \mathbf{A}\begin{bmatrix}\mathbf{0} \\ m\mathbf{g}\end{bmatrix}\label{eq:bilinear_constraint}
    \end{gather}
    Noting that $m\mathbf{g}$ and $\mathbf{A}$ are fixed for a given ground contact, the only remaining nonconvexity is from the $\mathbf{S}(\dot{\mathbf{l}}_G)$ term.
    
    Note that the left hand side of (\ref{eq:bilinear_constraint}) is bilinear in $[\mathbf{x}_{task}^T ~ \mathbf{u}_{task}^T]^T$ but linear in $\dot{\mathbf{l}}_G$ alone. Furthermore, the constraint $\|\dot{\mathbf{l}}_G\|_{\infty} \leq \dot{l}_{max}$ defines a polytope (specifically, a cube in $\mathbb{R}^3$). Linear functions constrained to polytopes have extrema at the vertices, so we can create a linear inner approximation of 
    (\ref{eq:bilinear_constraint}) by enforcing (\ref{eq:bilinear_constraint}) for $\dot{\mathbf{l}}_G$ evaluated at all 8 corners of the cube defined by $\|\dot{\mathbf{l}}_G\|_{\infty} \leq \dot{l}_{max}$.
    
    This new constraint is linear in the position of the CoM $\mathbf{p}_G$ and the time derivative of the centroidal momentum $\dot{\mathbf{h}}_G$, allowing us to write it as a constraint of the form
    \begin{equation}\label{eq:linear_constraint}
        \mathbf{A}_{cwc}\begin{bmatrix}\mathbf{x}_{task} \\ \mathbf{u}_{task}\end{bmatrix} \leq \mathbf{b}_{cwc}.
    \end{equation}
\end{proof}

This linear encoding of the CWC criterion allows us to perform MPC for the template model by solving a QP. To account for contact constraints, we consider the anchor variables $\mathbf{x}_{task}, \mathbf{u}_{task}$ as additional optimization variables, and enforce the interface (\ref{eq:interface}) as a constraint. We then perform MPC using a simple forward Euler direct collocation scheme as follows:
\begin{align}
    \min &~ 
    \sum_{t=1}^{N-1} \|\mathbf{x}_{lip}^t\|^2_{\mathbf{Q}_{mpc}} + \|u_{lip}^t\|^2_{R_{mpc}} + \|\mathbf{x}_{lip}^N\|_{\mathbf{Q}_f}\label{eq:mpc_qp} \\
    \text{s.t.} &~ \mathbf{x}_{lip}^0, \mathbf{x}_{task}^0 \text{ given} \label{eq:initial_cond}\\
    &~ \mathbf{x}_{lip}^{t+1} = \mathbf{x}_{lip}^t + (\mathbf{A}_{lip}\mathbf{x}_{lip} + \mathbf{B}_{lip}u_{lip})dt \label{eq:lip_euler}\\
    &~ \mathbf{x}_{task}^{t+1} = \mathbf{x}_{task}^t + (\mathbf{A}_{task}\mathbf{x}_{task} + \mathbf{B}_{task}\mathbf{u}_{task})dt \label{eq:anchor_euler}\\
    &~ \mathbf{u}_{task}^t = \mathbf{R}u_{lip}^t + \mathbf{Q}\mathbf{x}_{lip}^t + \mathbf{K}(\mathbf{x}_{task}^t - \mathbf{x}_{lip}^t) \label{eq:interface_constraint}\\
    & \mathbf{A}_{cwc}\begin{bmatrix}\mathbf{x}^t_{task} \\ \mathbf{u}^t_{task}\end{bmatrix} \leq \mathbf{b}_{cwc} \label{eq:contact_constraint} \\
    & \|\dot{\mathbf{l}}_G\|_\infty \leq \dot{l}_{max} \label{eq:accel_bound},
\end{align}
where (\ref{eq:initial_cond}) fixes the initial conditions, (\ref{eq:lip_euler}-\ref{eq:anchor_euler}) enforce forward Euler dynamic constraints, (\ref{eq:interface_constraint}) enforces the feasibility of the interface, and (\ref{eq:contact_constraint}-\ref{eq:accel_bound}) ensures contact constraints are met in the anchor model.

\begin{remark}
    The constraints (\ref{eq:contact_constraint}-\ref{eq:accel_bound}) represent an inner approximation of the CWC, and as such, involve a tradeoff with the parameter $\dot{l}_{max}$. If $\dot{l}_{max}$ is too high, the intersection of constraints defined by (\ref{eq:bilinear_constraint}) will be negligible or empty. On the other hand, if $\dot{l}_{max}$ is too small, $\mathbf{u}_{task}$ may not be able to meet the interface constraint (\ref{eq:interface_constraint}). Similarly, while we can find simulation relations with arbitrarily high decay rates $\lambda$, the resulting large $\mathbf{K}$ may cause a conflict between the interface constraint (\ref{eq:interface_constraint}) and the CWC constraints (\ref{eq:contact_constraint}-\ref{eq:accel_bound}). 
\end{remark}

These constraints allow us to plan using the (lower-order) LIP model in an MPC fashion. The fact that the constraints are linear means that planning is as simple as solving a QP, for which many fast solvers exist. Then, when the anchor model tracks the nominal LIP trajectory using the interface (\ref{eq:interface}), tracking is guaranteed to be $\epsilon$-close by Theorem \ref{theorem:bound}, as long as additional constraints on torques, joint limits, self-collisions, etc. are met. 

\section{Simulation Results}\label{sec:example}

As an example, we control the balancer shown in Figure~\ref{fig:models}. We assume the balancer is constrained to the ($x,z$) plane, so $\mathbf{x}_{task} \in \mathbb{R}^5$, $\mathbf{u}_{task} \in \mathbb{R}^3$. All links except the ``torso'' are 1m long, uniform density, and have mass of 1kg. The foot was a square platform with length 1m and mass 5kg. The coefficient of friction was 0.3. The torso link has length 2m and mass 2kg. For the LIP model, we assume a height of $h=1.75$m. 

We simulated the balancer in Gazebo \cite{koenig2004design}, interfacing with Matlab via ROS \cite{Quigley09}. We used Casadi \cite{Andersson2018} and the qpOASES solver \cite{Ferreau2014} to solve (\ref{eq:mpc_qp}). All computation was performed on a laptop with an Intel i7 processor and 32GB RAM.

For the interface, we chose $\mathbf{K}$ by solving the following infinite horizon LQR problem:
\begin{align*}
    \min_{\mathbf{u}_{task}} & \int_0^\infty (\mathbf{x}_{task}^T\mathbf{x}_{task} + 0.01\mathbf{u}^T_{task}\mathbf{u}_{task}) dt \\
    \text{s.t. } & \dot{\mathbf{x}}_{task} = \mathbf{A}_{task}\mathbf{x}_{task} + \mathbf{B}_{task}\mathbf{u}_{task}.
\end{align*}
We then set $\lambda = 0.1$ and solved the SDP described in Section \ref{subsec:linear_simulation} to find $\mathbf{M}$. To linearize the CWC constraint, we chose $\dot{l}_{max} = 5N$.

We compared this approach with a standard QP for tracking the template model as per \cite{herzog2014balancing,escande2014hierarchical,wensing2013generation}:
\begin{align}\label{eq:qp}
    \min & \|\mathbf{J}_{com}\ddot{\mathbf{q}}+\dot{\mathbf{J}}_{com}\dot{\mathbf{q}} - \mathbf{u}_{com}\|_2^2 + w\|\ddot{\mathbf{q}}-\ddot{\mathbf{q}}^{des}\|_2^2\\
    \text{s.t. } & \mathbf{H}\ddot{\mathbf{q}} + \mathbf{C}\dot{\mathbf{q}} + \bm{\tau}_g = \bm{\tau} \\
    & \mathbf{f}_{0} \in CWC \\
    & \bm{\tau}_{min} \leq \bm{\tau} \leq \bm{\tau}_{max},
\end{align}
where $\mathbf{J}_{com}$ is the CoM jacobian and $\mathbf{u}_{com}$ is a desired CoM acceleration, which tracks a nominal template trajectory (see Section \ref{subsec:approximate_sim_for_balancer}). The secondary objective, weighted by $w=0.1$, is determined by a desired joint-space acceleration $\ddot{\mathbf{q}}^{des}$, which regulates the robot to the static position shown in Figure \ref{fig:balancer}.   

\begin{figure*}
    \centering
    \begin{subfigure}{0.24\linewidth}
        \centering
        \includegraphics[width=\linewidth]{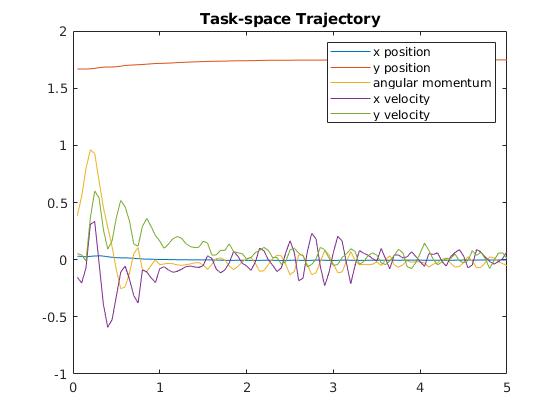}
        \caption{QP Approach, 20N Push}
    \end{subfigure}
    \begin{subfigure}{0.24\linewidth}
        \centering
        \includegraphics[width=\linewidth]{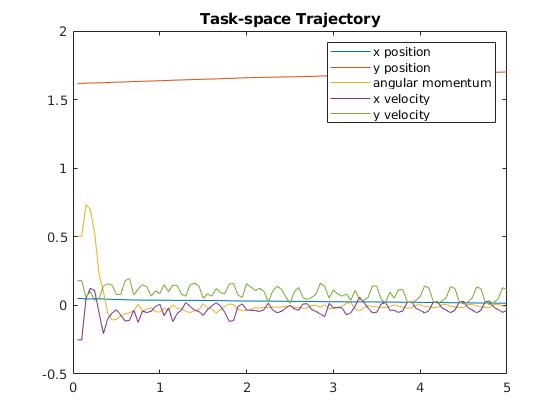}
        \caption{Our Approach, 20N Push}
    \end{subfigure}
    \begin{subfigure}{0.24\linewidth}
        \centering
        \includegraphics[width=\linewidth]{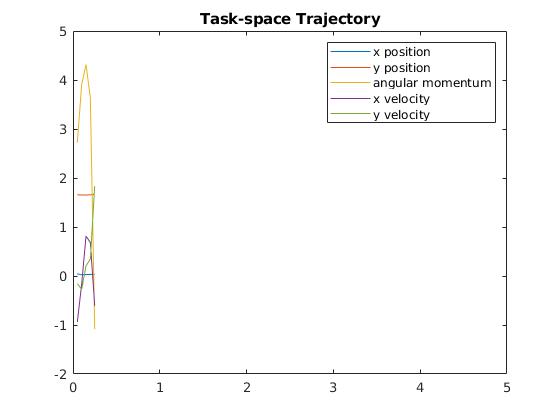}
        \caption{QP Approach, 100N Push}
    \end{subfigure}
    \begin{subfigure}{0.24\linewidth}
        \centering
        \includegraphics[width=\linewidth]{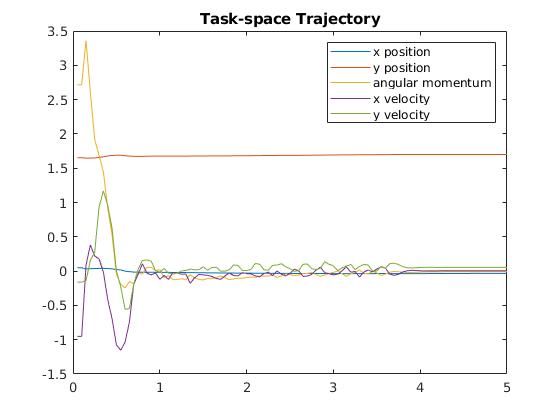}
        \caption{Our Approach, 100N Push}
    \end{subfigure}
    \caption{task-space trajectories for the push recovery scenario. A high angular momentum disturbance in the 100N scenario causes the traditional QP approach to fail, while our approximate simulation-based controller successfully balances the robot.}
    \label{fig:task_space_trajectories}
\end{figure*}

\begin{figure*}
    \centering
    \begin{subfigure}{0.24\linewidth}
        \centering
        \includegraphics[width=\linewidth]{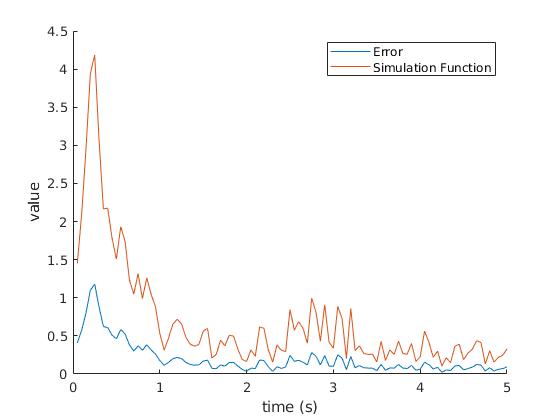}
        \caption{QP Approach, 20N Push}
    \end{subfigure}
    \begin{subfigure}{0.24\linewidth}
        \centering
        \includegraphics[width=\linewidth]{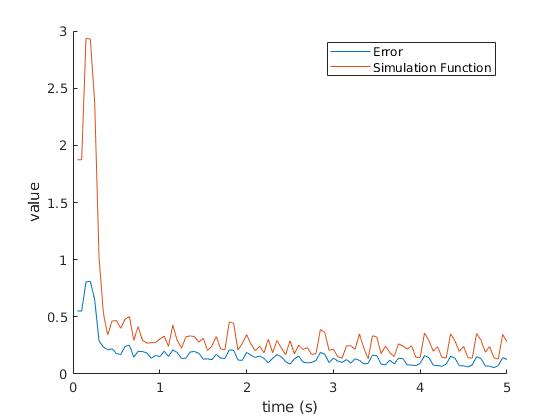}
        \caption{Our Approach, 20N Push}
    \end{subfigure}
    \begin{subfigure}{0.24\linewidth}
        \centering
        \includegraphics[width=\linewidth]{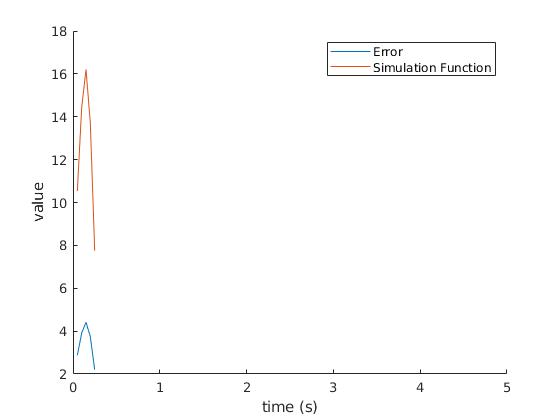}
        \caption{QP Approach, 100N Push}
    \end{subfigure}
    \begin{subfigure}{0.24\linewidth}
        \centering
        \includegraphics[width=\linewidth]{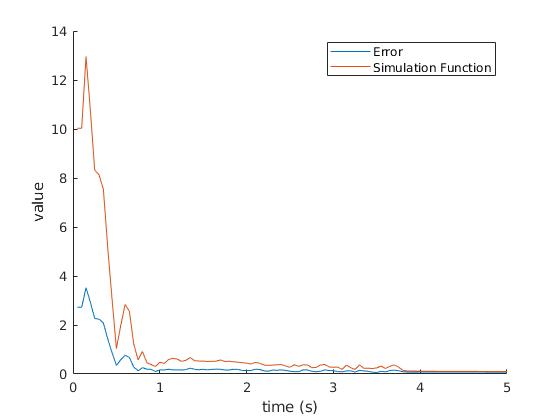}
        \caption{Our Approach, 100N Push}
    \end{subfigure}
    \caption{Output error and the simulation function (\ref{eq:simulation_fcn}) over time for the push-recovery scenario. For all scenarios, the simulation function bounds the output error. Using our approach, the simulation function is nonincreasing (apart from time-discretization error) after the push.}
    \label{fig:error_bounds}
\end{figure*}
Starting from an initial balanced state, a force was applied for 10ms to the top of the torso link in the $-x$ direction to simulate a push. For our approach, we performed MPC for the template model (\ref{eq:mpc_qp}) with $\mathbf{Q}_{mpc} = \text{diag}([10~0~0~10~0])$, $R_{mpc}=5.0$, and $\mathbf{Q}_f = 100\mathbf{Q}_{mpc}$. We applied commands and solved this MPC problem at 20Hz ($dt=0.05$) with horizon $N=5$. For the traditional QP approach, we computed a nominal control for the template model via LQR and solved (\ref{eq:qp}) at 20Hz as well.   

For a $20N$ push, both the QP approach and our approximate simulation approach successfully returned the balancer to an upright position. For a $100N$ push, however, the traditional QP (\ref{eq:qp}) became infeasible after 5 timesteps, causing the robot to fall down. Our approach successfully recovered from the $100N$ push. 

Task-space trajectories for these simulations are shown in Figure \ref{fig:task_space_trajectories}, and the associated errors and simulation functions are shown in Figure \ref{fig:error_bounds}. For all the scenarios, the simulation function (\ref{eq:simulation_fcn}) bounds the output error $\|\mathbf{x}_{task}-\mathbf{x}_{lip}\|$. However, this bound is tighter when using our approach, and apart from some noise due to time discretization, decreases over time. Snapshots of our approach recovering from a 100N push are shown in Figure \ref{fig:snapshots}. 


\section{Discussion}\label{sec:discuss}

Our primary result is showing a formal connection between template and anchor models, namely that a balancer in single support approximately simulates a LIP. The associated interface, interestingly, is a generalization of the partial feedback linearization-based PD controller often used to track template models. Controlling the anchor model with this interface brings us closer to providing formal guarantees regarding tracking performance. Specifically, we can compute a simulation function which bounds the output error between the two models as long as constraints on contacts, torques, joint limits, etc. are not violated. 

As a secondary result, we derived linear constraints that are sufficient conditions for maintaining ground contact. This allows for MPC planning with the template model by solving a QP, for which there are many fast solvers. While our simulation demonstrated solving this QP with a relatively short horizon ($N=5$), we expect that further optimization of the code and conversion to C/C++ will enable planning with longer horizons. 

Using this approach for control of a planar balancer enabled recovery from a large-magnitude push, which a standard QP controller was unable to recover from. While additional tuning of this controller and running it at a higher rate would likely improve performance, the fact that our controller explicitly accounts for angular momentum and contact constraints in the template planning phase suggests that it is more robust to angular momentum disturbances. Furthermore, this explicit accounting for angular momentum results in non-preprogrammed behavior of the ``arm'' link: the arm swings upward after the push, seemingly in at attempt to regulate angular momentum.

Finally, the particular interface (\ref{eq:interface}) that we propose certifies approximate simulation, but approximate simulation is a relationship between models. This means that there may be other control policies that also provide guaranteed tracking performance. An important open question is whether other control approaches, such as the QP constraint approach commonly used in whole-body control, can also be shown to provide such provably correct tracking of the template model. 

\begin{figure}
    \centering
    \includegraphics[width=0.6\linewidth]{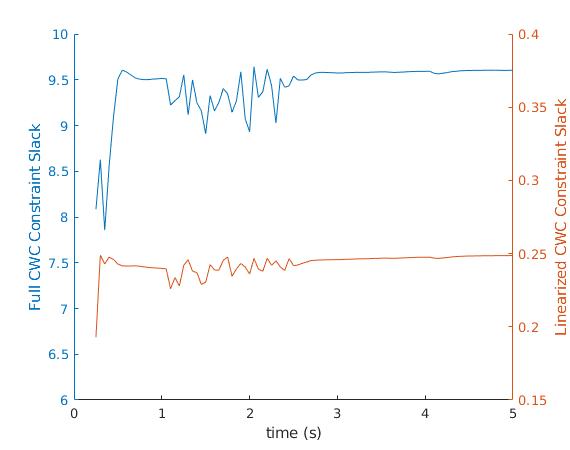}
    \caption{Slack for the full CWC constraint (\ref{eq:cwc_bilinear}) and the linearized CWC constraint (\ref{eq:contact_constraint}) following a 20N push. }
    \label{fig:constraint_slacks}
\end{figure}

\section{Conclusion}\label{sec:conclusion}

We explored approximate simulation as a means of providing formal connections between template and anchor models. We showed that a balancer in single support approximately simulates a linear inverted pendulum and derived the associated interface. We found that this interface is a generalization of the PD controller that is commonly used to track template models. As a secondary result, we derived linear constraints that are sufficient conditions for maintaining ground contact. These allow for rapid planning and replanning in the template model by solving a quadratic program. In a simulated push-recovery scenario for a planar balancer, our approximate simulation-based controller recovered from a large push disturbance that a conventional whole-body controller failed to recover from. Future work will extend these results to account for joint and torque limits, self-collisions, and multi-contact scenarios. 

\bibliographystyle{IEEEtranN}
\balance
{\footnotesize
\bibliography{references}}

\end{document}